\documentclass[runningheads]{llncs}
\usepackage[T1]{fontenc}
\usepackage{graphicx}
\usepackage{amsmath,amssymb,amsfonts}
\usepackage[]{color-edits}

            \DeclareMathOperator*{\minimize}{minimize}
		\DeclareMathOperator{\subto}{subject\,to}

		\newcommand{\st}{s.t.}

		\DeclareMathAlphabet{\mathsfit}{T1}{\sfdefault}{\mddefault}{\sldefault}
		
\usepackage{algpseudocodex}
\usepackage{algorithm}
\algnewcommand{\AlgAnd}{\textbf{and }}
\algnewcommand{\AlgInput}{\Statex \textbf{Input: }}
\algnewcommand{\AlgOutput}{\Statex \textbf{Output: }}
\tikzset{algpxIndentLine/.style={draw=none}}

\usepackage{graphicx}
 \usepackage{adjustbox} 
\usepackage{relsize}
\usepackage{multirow}
\usepackage{subcaption}

\usepackage{array}
\usepackage{tabularx}

\let\oldsubsection\subsection
\renewcommand{\subsection}[1]{
  \vspace{-8pt}
  \oldsubsection{#1}
  \vspace{-3pt}
}

\let\oldsubsubsection\subsubsection
\renewcommand{\subsubsection}[1]{
  \vspace{-8pt}
  \oldsubsubsection{#1}
}

\usepackage{setspace}

\usepackage{xcolor}
\usepackage{tikz}
\usetikzlibrary{fit,calc}

\newcommand*{\tikzmk}[1]{\tikz[remember picture,overlay] \node (#1) {};\ignorespaces}

\newcommand{\boxita}[1]{%
    \tikz[remember picture,overlay]{%
        \node[yshift=2.5pt,xshift=0.2cm,draw=#1,thin,rounded corners,fit={(A)($(B)+(-0.05\linewidth,.2\baselineskip)$)}] {};%
    }\ignorespaces
}

\newcommand{\boxitcb}[1]{%
    \tikz[remember picture,overlay]{%
        \node[yshift=2.5pt,xshift=0.2cm,draw=#1,thin,rounded corners,fit={(A)($(B)+(-12pt,.2\baselineskip)$)}] {};%
    }\ignorespaces
}
\newcommand{\boxitb}[1]{%
    \tikz[remember picture,overlay]{%
        \node[yshift=-10pt,xshift=-1.55cm,draw=#1,thin,rounded corners,fit={(A)($(B)+(1\linewidth,.8\baselineskip)$)}] {};%
    }\ignorespaces
}

\newcommand{\boxitd}[1]{%
    \tikz[remember picture,overlay]{%
        \node[yshift=-10pt,xshift=-1.55cm,draw=#1,thin,rounded corners,fit={(A)($(B)+(.85\linewidth,.8\baselineskip)$)}] {};%
    }\ignorespaces
}

\colorlet{mypink}{red}
\colorlet{myblue}{cyan}

\usepackage{caption}
\makeatletter
\newcommand\thefontsize[1]{{#1 The current font size is: \f@size pt\par}}
\makeatother

\usepackage{footnotebackref}

\usepackage[]{cleveref}
\Crefname{figure}{Fig.}{Figs.}
\Crefname{table}{Table}{Tables}
\Crefname{equation}{Eq.}{Eqs.}
\Crefname{section}{Sec.}{Secs.}
\Crefname{algorithm}{Alg.}{Algs.}
\Crefname{definition}{Def.}{Defs.}
\Crefname{appendix}{Appendix}{Appendices}

\usepackage{hyperref}
\hypersetup{
    colorlinks = true,   
    linkcolor = blue,    
    citecolor = blue,    
    urlcolor = blue,      
    hyperfootnotes = true,
}
\addauthor{sy}{teal}
\addauthor{bpg}{brown}
\addauthor{rt}{purple}
\addauthor{lpk}{violet}
\addauthor{rj}{orange}

\begin{document}
\title{GCS*: Forward Heuristic Search on Implicit Graphs of Convex Sets}
%
%
\author{Shao Yuan Chew Chia \and
Rebecca H. Jiang \and
Bernhard Paus Graesdal \and \\
Leslie Pack Kaelbling \and
Russ Tedrake
}
\authorrunning{S.Y. Chew Chia et al.}
%
\institute{Massachusetts Institute of Technology, Cambridge, MA 02139, USA\\
\email{\{shaoyuan, rhjiang, graesdal, lpk, russt\}@mit.edu}}
\vspace{-15pt}
\maketitle              
\begin{abstract}
\vspace{-15pt}
We consider large-scale, implicit-search-based solutions to Shortest Path Problems on Graphs of Convex Sets (GCS).
We propose GCS*, a forward heuristic search algorithm that generalizes A* search to the GCS setting, where a continuous-valued decision is made at each graph vertex, and constraints across graph edges couple these decisions, influencing costs and feasibility. 
Such mixed discrete-continuous planning is needed in many domains, including motion planning around obstacles and planning through contact.
This setting provides a unique challenge for best-first search algorithms: 
the cost and feasibility of a path depend on continuous-valued points chosen along the entire path.
We show that by pruning paths that are cost-dominated over their entire terminal vertex, GCS* can search efficiently while still guaranteeing cost-optimality and completeness.
To find satisficing solutions quickly, we also present a complete but suboptimal variation, pruning instead reachability-dominated paths.
We implement these checks using polyhedral-containment or sampling-based methods. 
The former implementation is complete and cost-optimal, while the latter is probabilistically complete and asymptotically cost-optimal and performs effectively even with minimal samples in practice.
We demonstrate GCS* on planar pushing tasks where the combinatorial explosion of contact modes renders prior methods intractable and show it performs favorably compared to the state-of-the-art.
Project website: \href{https://shaoyuan.cc/research/gcs-star/}{shaoyuan.cc/research/gcs-star/}

\end{abstract}
\vspace{-20pt}
\section{Introduction}
\label{sec:introduction}
\looseness=-1Many real-world planning problems involve making discrete and continuous decisions jointly.  Collision-free motion planning selects whether to go left or right around an obstacle along with a continuous trajectory to do so. In Task and Motion Planning (TAMP), discrete task-level decisions about the type and sequence of actions are intimately coupled with continuous robot motions and object configurations. For example, where a robot grasps a hockey stick impacts its ability to hook an object \cite{maoLearningReusableManipulation2023}; 
in the construction of a tower, the order in which materials are assembled, as well as their geometric relationships, 
affect stability \cite{hartmannLongHorizonMultiRobotRearrangement2023,toussaintLogicGeometricProgrammingOptimizationBased2015}.
\begin{figure}[h]
    \begin{subfigure}[t]{1.0\linewidth}
	\centering
	\includegraphics[trim=0 40 0 0, clip, width=1\linewidth]{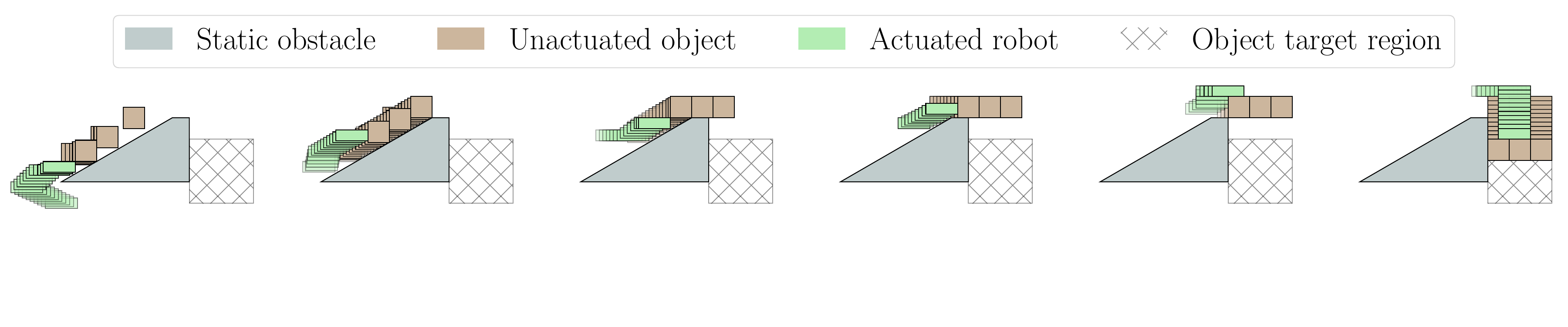}
    \end{subfigure}\vspace{-0.5cm}
    \caption{
        An $\epsilon$-suboptimal solution found in 21.9s by GCS* using sampling-based $\textsc{ReachesCheaper}$ domination checks on the STACK planar pushing task.
        STACK is formulated as a GCS problem with approximately $1.3 \times 10^{9}$ vertices and up to $8.5 \times 10^{17}$ edges.\vspace{-0.25cm}
    }
    \label{fig:hero}
\end{figure}

A natural representation of such a problem structure is the Graph of Convex Sets (GCS) \cite{marcucciShortestPathsGraphs2024}.  In a GCS, graph vertices represent these discrete choices, with edges denoting allowed transitions.  A continuous decision within each vertex is encoded by a \textit{point} constrained to lie within a convex set associated with the vertex.  For example, the free configuration space around an obstacle may be (sometimes approximately) decomposed into a number of intersecting convex regions in which trajectory segments can be planned.
The point associated with each vertex, which encodes the motion within the corresponding free-space region, could be, for example, a sequence of positions at a fixed number of knot points \cite{marcucciMotionPlanningObstacles2023}. 
In the case of a single knot point, the associated vertex's convex set is simply the corresponding convex free-space set.  In the case of more knot points, the vertex's set would be the higher-dimensional set in which each knot point is contained within the free-space set.  
In this example, intersecting free-space regions lead to graph edges. A valid way to get around the obstacle involves a \textit{path} -- a sequence of vertices -- through the graph, along with a \textit{trajectory} -- a sequence of continuous-valued points assigned to vertices on the path.
Importantly, in such problems, there may be constraints coupling points across edges, presenting a fundamental difference from classical graph search problems.  For example, for trajectory continuity, knot points across an edge must coincide.

\looseness=-1While Shortest-Path Problems (SPPs) in discrete graphs can be solved in polynomial time, SPPs in GCS are NP-hard \cite{marcucciShortestPathsGraphs2024}.  GCS can be solved with a Mixed-Integer Convex Program (MICP) formulation; Marcucci et al. give a notably scalable transcription of SPPs in GCS with a tight convex relaxation \cite{marcucciShortestPathsGraphs2024}.  However, these approaches become intractable as the graph grows.
Many discrete-continuous problems exhibit a combinatorial explosion of discrete modes, e.g., in a manipulation task where contact dynamics between every pair of bodies are considered. In TAMP, the search space is typically so large that explicit enumeration of the search space is impractical.
\begin{figure}[h]
    \centering
    \begin{subfigure}[b]{0.4\linewidth}
        \centering
        \raisebox{10pt}{
        \includegraphics[trim=0 0 0 0, clip, width=\textwidth]{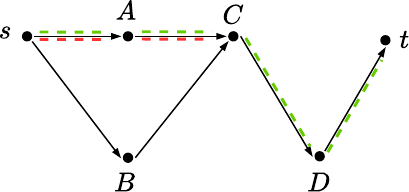}}\vspace{-0.2cm}
        \caption{}
    \end{subfigure}
    \hspace{0.05\textwidth}
    \begin{subfigure}[b]{0.4\linewidth}
        \centering 
        \includegraphics[trim=50 0 50 10, clip, width=\textwidth]{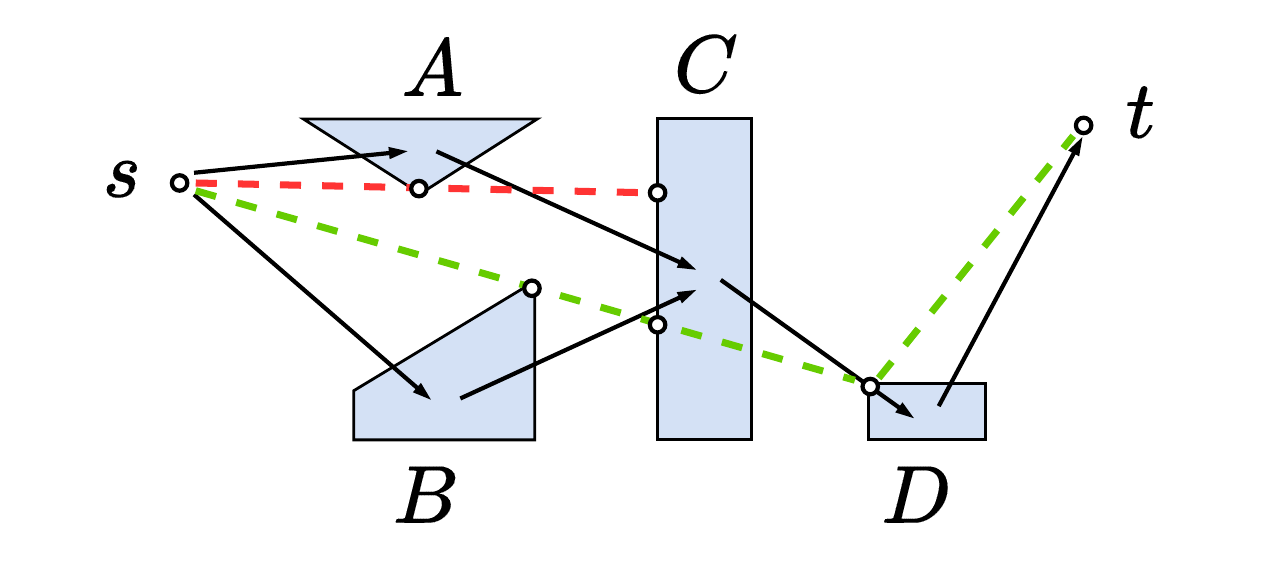}\vspace{-0.2cm}
        \caption{}
    \end{subfigure}
\caption{An abstract discrete graph (a) and GCS (b), for example, mapping to the path planning problem for a hopping robot tasked to hop from $s$ to $t$ across polygonal stepping stones.
    Arrows are edges between vertices. Vertices are represented as black dots in (a) and via their (blue) convex sets in (b). White dots are trajectory points on the GCS.
    Edge costs are the Euclidean distances traversed between vertices.
    A discrete graph (a) has the optimal substructure property while a GCS (b) does not.
    In (a), the optimal path between $s$ and $C$ (red) is a sub-path of the optimal path between $s$ and $t$ (green), which goes through $C$. In (b), the optimal path between $s$ and $C$ is not a sub-path of the optimal path between $s$ and $t$, which goes through $C$.  However, each \textit{point} on the green path's trajectory is reached optimally on the green path.\vspace{-0.22cm}
}
\label{fig:optimal_substructure}
\end{figure}
For such problems, we can instead define a graph implicitly via a source vertex and a successor operator that acts on a vertex and returns all outgoing edges. Implicit graph search algorithms use this representation to solve SPPs without loading the entire graph in memory.
A* \cite{hartFormalBasisHeuristic1968} is forward heuristic search algorithm which can solve the SPP on implicit discrete graphs.
However, key differences between discrete graphs and GCS make search in GCS more challenging.

\looseness=-1 The SPP on a discrete graph has the optimal substructure property \cite{cormenIntroductionAlgorithms2022}: All sub-paths of an optimal path between two vertices are themselves optimal paths between the first and last vertices they traverse. 
A* (or any best-first search algorithm) leverages this property, pruning a sub-path when its cost-to-come is dominated by another sub-path terminating in the same vertex because it cannot be a sub-path of any optimal path \cite{russellArtificialIntelligenceModern2021}.
However, the SPP on a GCS lacks the optimal substructure property. Importantly, as exemplified in \cref{fig:optimal_substructure} and by Morozov et al. \cite{morozovMultiQueryShortestPathProblem2024}, \textit{a sub-path of an optimal path on a GCS is not necessarily an optimal path between that subpath's first and last vertices}. Thus, when considering whether a partial path to the target could be a sub-path of an optimal path, it is insufficient to ask whether it is the cheapest way to reach its terminal vertex. 
As a result, a naive application of A* to GCS could prune an important sub-path, preventing the optimal path from being returned.
We leverage the key insight that, while the optimal substructure property does not hold on the path level in GCS, it does hold on the trajectory level: Given an optimal path and corresponding trajectory, any sub-path is optimal to reach its own final point, as exemplified in \cref{fig:optimal_substructure} (b). 
That is, while an optimal path may not contain an optimal subpath to reach each \textit{vertex} it visits, its trajectory is optimal to reach each \textit{point} it visits.
In order to maintain this property, we place no restrictions on the number of times a path may revisit vertices \cite{morozovMultiQueryShortestPathProblem2024}. 
Formally, paths that may revisit vertices are generally referred to as ``walks,'' but for readability we continue to use  ``path'' even though this term traditionally excludes cycles.

\looseness=-1 The trajectory-level optimal substructure property holds because the optimal way to reach a point is independent of the optimal way to continue on from it.  
However, this is only true if visiting a particular vertex on the way to that point does not exclude that vertex from being visited subsequently.  
It is well-known that shortest paths on discrete graphs never revisit vertices \cite{cormenIntroductionAlgorithms2022}; however, this does not hold on a GCS because a different point may be visited each time.

More severe than the loss of optimality is the loss of completeness. Because a GCS may have constraints coupling points across edges, a path may not be feasible even if edges connect the sequence of vertices.
As a result, A* on a GCS might prune a candidate path even though the cheaper alternate path is infeasible for reaching the target.
An example is shown in \cref{fig:lose_completeness}.

\begin{figure}[h]
	\centering
	\includegraphics[trim=0 5 0 5, clip, width=0.55\linewidth]{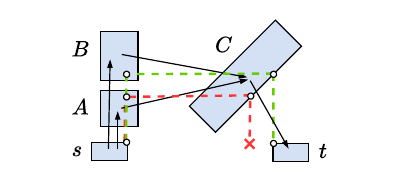}\vspace{-0.2cm}
	\caption{
        A contrived example of a GCS where points spanning edges $(s,A)$, $(s,B)$, and $(C,t)$ are constrained to be vertically aligned, and points spanning $(A,C)$ and $(B,C)$ horizontally aligned. 
        For example, representing a path planning problem for a hopping robot that can only hop in axis-aligned directions.
        A* prunes $[s,B,C]$ since it has higher cost-to-come than $[s,A,C]$. However, $[s,A,C,t]$ is infeasible (red), while $[s,B,C,t]$ is feasible (green).\vspace{-0.2cm}
	}
	\label{fig:lose_completeness}
\end{figure}

In this paper, we address these challenges and generalize A* to the GCS setting. We call our algorithm GCS* (\cref{sec:method}), leveraging two key insights. First, to retain optimality, the algorithm must keep track of every path that reaches \textit{any} point in a set more cheaply than alternate paths. 
Second, to retain completeness, the algorithm only needs to keep track of every path that reaches previously unreached points within a set.
These insights give rise to domination checks we call \textsc{ReachesCheaper} and \textsc{ReachesNew}, respectively (\cref{sec:domination_checks}). Using \textsc{ReachesCheaper} with an admissible heuristic makes GCS* optimal, while using \textsc{ReachesNew} prunes paths more aggressively, sacrificing optimality for speed, but retaining completeness\footnote{Similar to A* \cite{russellArtificialIntelligenceModern2021}, GCS* is complete if a solution exists or the user defines a limit on the length of paths that can be returned.}. 
We present polyhedral-containment-based implementations of these checks and prove completeness and optimality of GCS* (\cref{sec:algorithm_properties}).  However, we achieve significant speed improvements with only mildly suboptimal results with sampling-based implementations, under which GCS* is probabilisitically complete and asymptotically cost optimal, i.e., complete and optimal in the limit of infinite samples.
The theoretical properties of each variant of the algorithm are summarized in \cref{tab:task_results}.

Finally, we formulate planar pushing as a GCS problem in \cref{sec:planar_pushing} and share GCS* results for tasks with multiple moving objects and multi-step solutions in \cref{sec:results}.  Such problems are notorious for combinatorially exploding discrete options.  
For example, the STACK task (\cref{fig:hero}) leads to approximately $10^9$ sets and up to $10^{18}$ edges, making it wildly intractable for methods that require loading the graph into memory, but sampling-based GCS* finds a solution in 21.9 seconds.
Two key strengths of the best-first search framework extend to GCS*. First, GCS* is able to solve huge problems that are intractable for methods which require explicit construction of the graph.
Second, search effort is proportional to query difficulty.
In contrast, methods that solve the problem as a single optimization program require full effort regardless of query difficulty.
Similar to A*, GCS* also guides the search using a priority function that sums the cost-to-come and heuristic cost-to-go. Exactly as in weighted A*, inflation of the heuristic can be used to trade-off bounded sub-optimality with speed \cite{pohlHeuristicSearchViewed1970}.

\vspace{-15pt}
\section{Related Work}
\label{sec:related_work}
\vspace{-10pt}
\looseness=-1 We overview three approaches to discrete-continuous planning: sampling-based, optimization-based, and hybrid methods that combine graph search with sampling or optimization.
Sampling-based methods, like  Rapidly-exploring Random Trees \cite{lavalleRapidlyexploringRandomTrees1998} and Probabilistic Roadmap Method \cite{kavrakiProbabilisticRoadmapsPath1996} were developed for path planning with obstacles and extended for contact-rich manipulation \cite{chengEnhancingDexterityRobotic2023,pangGlobalPlanningContactRich2023}.
Unlike nonconvex trajectory optimization, which can fail to find feasible solutions, sampling-based methods are typically probabilistically complete. However, incorporating continuous differential constraints on discrete samples is challenging, often making kinodynamic versions of these algorithms less effective \cite{marcucciMotionPlanningObstacles2023}.

\looseness=-1To achieve both the completeness of sampling-based methods and trajectory optimization's ease of handling continuous constraints, some have proposed posing the entire problem as a single MICP, solving a convex relaxation, and applying a cheap rounding step to recover a feasible solution. This strategy has been applied to motion planning around obstacles for high-dimensional robotic systems \cite{marcucciMotionPlanningObstacles2023} and planar pushing of a single object \cite{graesdalTightConvexRelaxations2024}. We extend these benefits to the class of  problems too large be posed as a single optimization program.

\looseness=-1To solve large problems modelled implicitly,
many TAMP algorithms perform a tree search, alternating between choosing a discrete action and sampling continuous action parameters \cite{garrettIntegratedTaskMotion2021}.
However, using sampling to determine feasibility of a sequence of abstract actions does not provide a guaranteed answer in finite time. Other hybrid methods combine graph search with optimization over continuous variables. Using optimization to set path expansion priority allows continuous considerations to closely inform discrete search.
Logic-Geometric Programming \cite{toussaintLogicGeometricProgrammingOptimizationBased2015,toussaintMultiboundTreeSearch2017} uses non-linear programming (NLP) to guide the search, but the non-convexity of NLP leads to the loss of completeness and optimality.
We build on these ideas, additionally requiring a convex decomposition of configuration space. While this can be non-trivial to generate \cite{petersenGrowingConvexCollisionFree2023,daiCertifiedPolyhedralDecompositions2023,wernerApproximatingRobotConfiguration2024}, convex optimization allows us to determine the feasibility of a path in GCS absolutely, and perform domination checks that enable completeness and global optimality.

Recent works have proposed solving SPPs on GCS with graph search. $A^*$-$GCS$ \cite{sundarGraphsConvexSets2024} solves the convex relaxation on increasingly large sub-graphs, using successive solutions to guide sub-graph growth. In this work, we instead iteratively find solutions for sub-paths rather than sub-graphs, which scales better to larger graphs.
INSATxGCS (IxG) and IxG* \cite{natarajanImplicitGraphSearch2024} adapt INterleaved Search and Trajectory Optimization (INSAT) \cite{natarajanInterleavingGraphSearch2021} for collision-free GCS motion planning \cite{marcucciMotionPlanningObstacles2023}. These algorithms demonstrate impressive speedups over the convex relaxation and rounding approach \cite{marcucciMotionPlanningObstacles2023} across a range of tasks including multi-arm assembly. However, IxG is only complete subject to the assumption that all edge constraints along any path in the graph can be satisfied; In reality, individually feasible edge constraints can still be jointly infeasible, as in \cref{fig:lose_completeness}.
Edge constraints add significant expressivity to GCS and are heavily used, including by Marcucci et al. \cite{marcucciMotionPlanningObstacles2023} and Graesdal et al. \cite{graesdalTightConvexRelaxations2024}.
While IxG* is complete without this assumption, it uses IxG to generate an upper bound on the optimal cost. Supposing IxG fails to find a feasible path, IxG* will use a trivial (infinite) upper bound. Then IxG* will not prune \emph{any} paths, rendering large problems intractable.
Furthermore, while the main search of IxG and IxG* can be done on implicit graphs, both methods use a precomputed heuristic that requires explicit graphs. We build on the this approach, addressing the case where graphs are too large to be explicitly constructed, e.g. the STACK task (\cref{fig:hero}).

\section{Problem Formulation}
\label{sec:problem_formulation}
\vspace{-10pt}
In this section we present the problem of interest as a mathematical program.
\begin{table}[h]
    \centering
    \vspace{-15pt}
    \caption{Summary of mathematical notation \vspace{.1cm}}
    \label{tab:notation}
    \footnotesize
    \begin{tabularx}{\linewidth}{@{\hspace{8pt}}l@{\hspace{8pt}}@{\hspace{8pt}}X@{\hspace{8pt}}}
        \hline
        $G := (\mathcal{V},\mathcal{E})$ & Directed graph with vertices $\mathcal{V}$ and edges $\mathcal{E}$ \\
        $\mathcal{X}_v$ & Convex set associated with vertex $v$ \\
        $\mathcal{X}_e$ & Convex constraint set for edge $e$ \\
        $s, t$ & Source and target vertices \\
        $\mathbf{v}$ & Discrete path in the graph, defined as a sequence of vertices \\
        $\mathbf{v}_i$ & The $i$-th vertex in the path $\mathbf{v}$ \\
        $\mathbf{v}_\text{end}$ & Last vertex in the path $\mathbf{v}$ \\
        $\mathcal{E}_{\mathbf{v}}$ & Set of edges traversed by path $\mathbf{v}$ \\
        $\mathbf{x}$ & Continuous trajectory\\
        $\mathbf{x}_i$ & The $i$-th point on trajectory $\mathbf{x}$\\
        $\mathbf{x}_i\in \mathcal{X}_v$ & The point on trajectory $\mathbf{x}$ associated with vertex $v$\\
        $c(\mathbf{x}_u, \mathbf{x}_v)$ & Cost of the edge between vertices $u$ and $v$, evaluated on trajectory $\mathbf{x}$\\
        $f^*(v)$ & Optimal cost of a solution via vertex $v$ \\
        $g^*(x)$ & Optimal cost-to-come from $s$ to point $x$ \\
        $h^*(x)$ & Optimal cost-to-go from point $x$ to $t$ \\
        $\tilde{f}(\mathbf{v})$ & Total cost estimate via path $\mathbf{v}$ \\
        $\tilde{f}(\mathbf{v}, x)$ & Total cost estimate via path $\mathbf{v}$ and point $x \in \mathcal{X}_{\mathbf{v}_{\text{end}}}$ \\
        $\tilde{g}(\mathbf{v}, x)$ & Optimal cost-to-come from $s$ to point $x \in \mathcal{X}_{\mathbf{v}_{\text{end}}}$ via path $\mathbf{v}$ \\
        $\tilde{h}(x)$ & Heuristic estimate of cost-to-go from point $x$ to $t$ \\
        $\mathbf{x}^*(\mathbf{v})$ & Optimal trajectory via path $\mathbf{v}$ subject to the heuristic $\tilde{h}$\\
        \hline
    \end{tabularx}
    \vspace{-5pt}
\end{table}

A GCS \cite{marcucciShortestPathsGraphs2024} is a directed graph $G := (\mathcal{V},\mathcal{E})$ defined by a (potentially infinite) set of vertices $\mathcal{V}$ and edges $\mathcal{E} \subset \mathcal{V}^2$, where $(u,v)\in \mathcal{E}$ if the graph allows transitions from vertex $u$ to vertex $v$, i.e., $v$ is a \emph{successor} of $u$.
Each vertex $v \in \mathcal{V}$ is paired with a compact convex set $\mathcal{X}_v$. 
These sets may live in different spaces.
We assume each vertex has a finite number of successors. The graph is implicitly defined via source vertex $s \in \mathcal{V}$ and operator $\textsc{Successors}$ that, when applied to a vertex $u$, returns all successors $v_i$ of $u$.
We define a path $\bf{v}$ in graph $G$ as a sequence of vertices, where the bold font indicates a sequence. A subscripted index operates on a sequence to select the element at that index: $\mathbf{v}_{i}$ gives the $i$-th vertex in the sequence and each $\mathbf{v}_{i+1}$ is a successor of $\mathbf{v}_i$.
Likewise, we define a trajectory $\mathbf{x}$ on path $\mathbf{v}$ as a sequence of points in the sets associated with $\mathbf{v}$, $\mathbf{x}_i\in \mathcal{X}_{\mathbf{v}_i}$.  We also overload the indexing of $\mathbf{x}$ to go by vertex, such that $\mathbf{x}_{\mathbf{v}_i}:= \mathbf{x}_i$.
We denote the last vertex in the path $\bf{v}$ as $\mathbf{v}_\text{end}$, and the edges traversed by $\bf{v}$ as $\mathcal{E}_{\bf{v}}$.
The cost of edge $e:=(u,v)$ is determined by a proper, closed, convex, positive, bounded-away-from-zero function of its endpoints $c(\mathbf{x}_u, \mathbf{x}_v)$. The bounded-away-from-zero stipulation ensures finite optimal paths when cycles are permitted. 
Each edge can additionally have constraints $(\mathbf{x}_u,\mathbf{x}_v) \in \mathcal{X}_e$, where $\mathcal{X}_e$ is a closed convex set.
For source vertex $s \in \mathcal{V}$ and target vertex $t \in \mathcal{V}$, the SPP in GCS is \vspace{-0.2cm}
\begin{subequations}
\begin{align}
    \minimize_{\mathbf{v},\: \mathbf{x}} \quad & \sum_{(u,v) \in \mathcal{E}_{\mathbf{v}}} c(\mathbf{x}_u,\mathbf{x}_v) \label{eq:spp_gcs:objective} \\
    \subto \quad & \mathbf{v}_0=s, \: \mathbf{v}_\text{end}=t, \label{eq:spp_gcs:start_end}\\
            & \mathbf{v}_{i+1}\in \Call{Successors}{\mathbf{v}_i}, &&\forall i\in [0, \text{end}-1], \label{eq:spp_gcs:valid_path} \\
           \quad & \mathbf{x}_v \in \mathcal{X}_v, && \forall v \in \mathbf{v}, \label{eq:spp_gcs:point_in_set}\\
           \quad & (\mathbf{x}_u, \mathbf{x}_v) \in \mathcal{X}_e, && \forall e := (u, v) \in \mathcal{E}_{\mathbf{v}}. \label{eq:spp_gcs:valid_edges}
\end{align}
\label{eq:spp_gcs}\vspace{-0.5cm}
\end{subequations}

\noindent
The objective (\ref{eq:spp_gcs:objective}) is to minimize the total cost of traversing the path $\mathbf{v}$ in the graph.
Constraint (\ref{eq:spp_gcs:start_end}) enforces that path $\mathbf{v}$ starts at source vertex $s$, and ends at target vertex $t$. 
Constraint (\ref{eq:spp_gcs:valid_path}) enforces that each edge in the path $\mathbf{v}$ exists.
Constraint (\ref{eq:spp_gcs:point_in_set}) enforces that each point on the trajectory $\mathbf{x}$ lies within the convex set corresponding to its vertex.
Finally, constraint (\ref{eq:spp_gcs:valid_edges}) enforces that the continuous values of trajectory $\mathbf{x}$ satisfy all edge constraints along path $\mathbf{v}$.

If the discrete path $\mathbf{v}$ is fixed, prog. (\ref{eq:spp_gcs}) becomes convex and easy to solve; we call this program \textsc{ConvexRestriction}. Given $\mathbf{v}$, solving \textsc{ConvexRestriction} determines its corresponding optimal cost and optimal trajectory $\mathbf{x}$.

\vspace{-10pt}
\section{Approach}
\label{sec:method}
\subsection{Graph Search Formulation}
A* (\cref{alg:a_star}) searches over paths from the source, pruning some and expanding others based on estimated costs of extending to reach the target, until a path reaches the target and is returned. 
In particular, 
paths await expansion in a priority queue ${Q}$ ordered by total cost estimate $\tilde{f}$ (\cref{line:popped_from_Q}). Often, instead of specifying $\tilde{f}$ directly, a user equivalently provides  $\tilde{h}$, a heuristic that estimates costs-to-go from vertices to the target, and
$\tilde{f}$ is computed as the sum of heuristic $\tilde{h}$ and cost-to-come $\tilde{g}$. 
A path is pruned via a domination check if its cost-to-come $\tilde{g}$ is greater than that of the current cheapest path reaching the same vertex (\cref{line:domination_check}). This current cheapest path for each vertex is stored in a map $S$ (\cref{line:add_to_S}),
that maps each vertex to the current lowest-cost path.
Equivalently, $\tilde{f}$ can be compared since paths reaching the same vertex share an $\tilde{h}$ value.
We use the term ``domination check'' generically to refer to a function that determines whether a candidate path should be pruned or added to $Q$.
We call a path \textit{expanded} if it has been popped from $Q$ (\cref{line:popped_from_Q}).
We call a vertex \textit{expanded} if any path terminating at it has been expanded.

\looseness=-1GCS* (\cref{alg:gcs_star}) proceeds similarly. 
The key difference is the use of a different domination check $\textsc{NotDominated} \in$ $\{\textsc{ReachesCheaper}, \textsc{ReachesNew} \}$ (defined in \cref{sec:domination_checks}).
To facilitate these domination checks, GCS* maintains a map ${S}$ which maps each vertex to a set of un-pruned paths reaching the vertex. This is in contrast to A*, which stores a single path reaching each vertex. 
For any vertex $v \in \mathcal{V}$, we define $f^*(v)$ as the optimal cost of prog.\ (\ref{eq:spp_gcs}) with the additional constraint that vertex $v$ is on the path $\mathbf{v}$ (not necessarily as the terminal vertex $\mathbf{v}_\text{end}$), i.e., $v\in \mathbf{v}$, where path $\mathbf{v}$ is a decision variable in prog.\ (\ref{eq:spp_gcs}).
Let $g^*(x)$ be the optimal cost-to-come from source $s$ to point $x$ , and $h^*(x)$ be the optimal cost-to-go from point $x$ to target $t$, both infinite if infeasible.\footnote{More precisely, every ``function'' $l(x)$ that takes a point $x$ as input in this paper is actually a family of functions $\{l_v: \mathcal{X}_v \to \mathbb{R} \mid v \in V\}$ defined over each vertex, as these points may lie in different spaces.  These functions are defined to agree across vertices when these vertex sets share points, $l_{v_1}(x) = l_{v_2}(x)\ \forall x\in \mathcal{X}_{v_1}\cap \mathcal{X}_{v_2}$. For clarity, we drop the subscript and treat $l(x)$ as a single function.}   If point $x\in \mathcal{X}_v$ is on the optimal path through vertex $v$, we have $f^*(v) = g^*(x)+ h^*(x)$. 
Note that $f^*(x)$, $g^*(x)$, and $h^*(x)$ are well defined even if point $x$ is contained in (intersecting) sets corresponding to multiple vertices.

\looseness=-1In general, $f^*, g^*,$ and $h^*$ cannot be computed without solving prog.\ (\ref{eq:spp_gcs}). 
Instead, we define computationally viable functions $\tilde{f}, \tilde{g}$ and $\tilde{h}$.
The A* algorithm uses a heuristic function for estimating the cost-to-go from a vertex to the target. Extending this to GCS*, we use a heuristic function $\tilde{h}(x)$ that estimates the cost-to-go from a point $x$ to the target $t$.

Next, we define $\tilde{g}$ and $\tilde{f}$. No longer referring to the decision variables from prog.\ (\ref{eq:spp_gcs}), let $\mathbf{v}$
be a candidate partial path to the target $t$ such that $\mathbf{v}_0 = s$, but $\mathbf{v}_\text{end}$ is not necessarily $t$.  
Let $\mathbf{x}^*(\mathbf{v})$ denote an optimal trajectory through $\mathbf{v}$ subject to the heuristic $\tilde{h}$ assigning the cost from the final point in the optimal partial trajectory $\mathbf{x}^*(\mathbf{v})_\text{end}$ to target $t$.  
Then $\tilde{g}(\mathbf{v}, x)$ for $x\in \mathcal{X}_{\mathbf{v}_\text{end}}$ is the optimal cost-to-come from source $s$ to point $x$ via partial path $\mathbf{v}$.
Then, the total cost estimate via a partial path $\mathbf{v}$ is defined to be $\tilde{f}(\mathbf{v}):= \tilde{g}(\mathbf{v}, \mathbf{x}^*(\mathbf{v})_\text{end}) + \tilde{h}(\mathbf{x}^*(\mathbf{v})_\text{end})$.  

\noindent
\begin{minipage}[t]{0.49\textwidth}
\vspace{-0.4cm}
    \centering
    \begin{algorithm}[H]
    \caption{Discrete A*}
    \label{alg:a_star}
    \begin{algorithmic}[1]
    \AlgInput $s$, $t$, $\tilde{f}$, $\tilde{g}$, $\textsc{Successors}$
    \Statex \vspace{-2pt}
    \AlgOutput Path from $s$ to $t$ or \texttt{Fail}
    \State $\mathbf{v} = [s]$ 
    \State \tikzmk{A}$S \gets \{s: \mathbf{v}\}$\tikzmk{B}\boxita{gray} \Comment{\small{Map of paths}}
    \State \resizebox{\hsize}{!}{$Q \gets \text{priority queue ordered by } \tilde{f}(\mathbf{v})$}
    \State $Q.\Call{add}{\mathbf{v}}$
    
    \While{$Q \neq \emptyset$}
        \State $\mathbf{v} \gets$ $Q.\Call{Pop}{ }$
        \If{$\mathbf{v}_\text{end}$ = t} \Return $\mathbf{v}$
        \EndIf
        \ForAll{$v' \in \Call{Successors}{\mathbf{v}_\text{end}}$}
            \State $\mathbf{v}' = [\mathbf{v}, v']$
            \tikzmk{A}
            \If{$v' \notin S$ or $\tilde{g}(\mathbf{v}') < \tilde{g}(S[v'])$}
                \State $S[v'] = \mathbf{v}'$\tikzmk{B}\boxitb{gray}
                \State $Q.\Call{Add}{\mathbf{v}'}$
            \EndIf
        \EndFor
    \EndWhile
    \State \Return \texttt{Fail}
    \end{algorithmic}
    \end{algorithm}
    \vspace{0.05cm}
\end{minipage}
\hfill
\hspace{0.01\textwidth}
\begin{minipage}[t]{0.5\textwidth}\vspace{-0.4cm}
    \centering
    \begin{algorithm}[H]
    
    \begin{algorithmic}[1]
    \AlgInput $s$, $t$, $\tilde{f}$, $\textsc{NotDominated}$, $\textsc{Successors}$
    
    \AlgOutput Path from $s$ to $t$ or \texttt{Fail}
    \State $\mathbf{v} = [s]$
    \State
    \tikzmk{A}
    $S \gets \{s: \{\mathbf{v}\}\}$ \tikzmk{B}\boxitcb{gray}\Comment{\small{Map of sets of paths}}
    
    \State \resizebox{\hsize}{!}{$Q \gets \text{priority queue ordered by } \tilde{f}(\mathbf{v})$}\label{line:priority_queue}
    \State $Q.\Call{add}{\mathbf{v}}$
    
    \While{$Q \neq \emptyset$}
        \State $\mathbf{v} \gets$ $Q.\Call{Pop}{ }$\label{line:popped_from_Q}
        \If{$\mathbf{v}_\text{end}$ = t} \Return $\mathbf{v}$ \EndIf
        \ForAll{$v' \in \Call{Successors}{\mathbf{v}_\text{end}}$}
            \State $\mathbf{v}' = [\mathbf{v}, v']$
            \tikzmk{A}
            \If{$\Call{NotDominated}{\mathbf{v}', S[v']}$} \label{line:domination_check}
                \State $S[v'].\Call{Add}{\mathbf{v}'}$\label{line:add_to_S}\tikzmk{B}\boxitd{gray}
                \State $Q.\Call{Add}{\mathbf{v}'}$\label{line:add_to_Q}
            \EndIf
        \EndFor
    \EndWhile
    \State \Return \texttt{Fail}
    \end{algorithmic}
    \caption{GCS*}
    \label{alg:gcs_star}
    \end{algorithm}
\end{minipage}
To evaluate the total cost estimate for a path additionally restricted to pass through a specific point $x\in \mathcal{X}_{\mathbf{v}_\text{end}}$, we overload $\tilde{f}(\mathbf{v}, x) := \tilde{g}(\mathbf{v}, x) + \tilde{h}(x)$.
$\textsc{ConvexRestriction}$ can be used to evaluate
$\tilde{g}(\mathbf{v}, x)$ and  $\tilde{f}(\mathbf{v}, x)$ (used to evaluate sample points in \cref{sec:sampling_based}). Additionally, if $\tilde{h}$ is convex\footnote{More precisely, $\tilde{h}_v$ is convex for all $v\in \mathcal{V}$.}, $\textsc{ConvexRestriction}$ can be used to evaluate $\mathbf{x}^*(\mathbf{v})$ (used to return the final optimal trajectory) and $\tilde{f}(\mathbf{v})$ (used to prioritize the queue $Q$ in \cref{alg:gcs_star}). 

We assume heuristic $\tilde{h}(x)$ is nonnegative, i.e., $\tilde{h}(x) \geq 0, \: \forall x$. Furthermore, heuristic $\tilde{h}$ must be \textit{pointwise admissible} (\cref{def:pointwise_admissible}) in order for optimality guarantees (\cref{sec:algorithm_properties}) to hold. We extend the classical definition of admissibility over vertices to a pointwise definition:
\begin{definition}
    \label{def:pointwise_admissible}
    A heuristic function $\tilde{h}$ is pointwise admissible if
    \[
    \tilde{h}(x) \leq h^*(x), \quad \forall x \in \mathcal{X}_v, \quad \forall v \in \mathcal{V}.
    \]
\end{definition}

\vspace{-10pt}
\subsection{Domination Checks}
In \cref{line:domination_check} of \cref{alg:gcs_star}, GCS* uses one of two domination checks, \textsc{ReachesCheaper} or \textsc{ReachesNew}. In practice, exact checks are not tractable, so 
 we compute approximate checks.  An approximate \textsc{NotDominated} check is said to be \emph{conservative} if it never returns $\texttt{False}$ incorrectly, i.e., it never says a candidate path is dominated when it is not.  Under a conservative domination check, GCS* may track more candidate paths than necessary, but never overlooks an important path (which paths are ``important'' depends on the domination check being used and will be precisely stated in this section). Using exact or conservative \textsc{ReachesCheaper} checks, GCS* is cost optimal. Satisficing solutions (solutions which are feasible but likely suboptimal) can be found more quickly using \textsc{ReachesNew}. Using exact or conservative \textsc{ReachesNew} checks, GCS* is complete.

\begin{figure*}[h]
	\centering
	\begin{subfigure}[b]{0.3\linewidth}
			\centering
			\includegraphics[width=\linewidth]{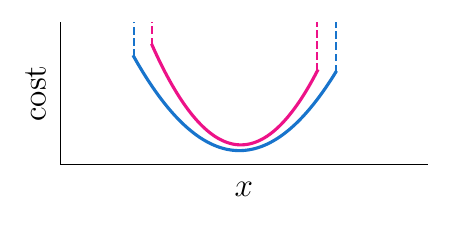}\vspace{-0.3cm}
       \caption{
       $\textsc{RC}=\texttt{False}, \: \textsc{RN}=\texttt{False}$
       }
    \end{subfigure}
    \centering
	\begin{subfigure}[b]{0.3\linewidth}
			\centering
			\includegraphics[width=\linewidth]{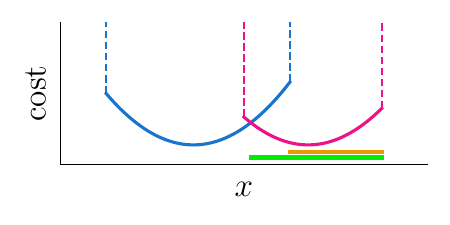}\vspace{-0.3cm}
       \caption{
           $\textsc{RC}=\texttt{True}, \: \textsc{RN}=\texttt{True}$
       }
    \end{subfigure}
    \centering
	\begin{subfigure}[b]{0.3\linewidth}
			\centering
			\includegraphics[width=\linewidth]{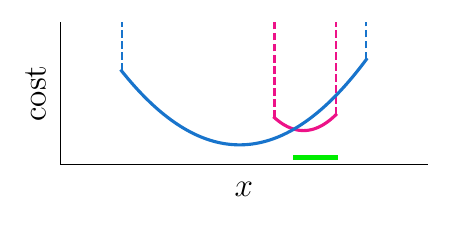}\vspace{-0.3cm}
       \caption{
       $\textsc{RC}=\texttt{True}, \: \textsc{RN}=\texttt{False}$
       }
   \end{subfigure}
    \\ 
    \begin{subfigure}[b]{0.3\textwidth}
        \centering
        \includegraphics[width=\linewidth]{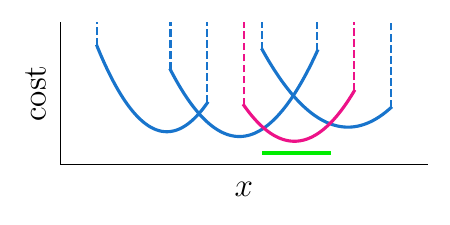}\vspace{-0.3cm}
        \caption{
           $\textsc{RC}=\texttt{True}, \: \textsc{RN}=\texttt{False}$
       }
    \end{subfigure}
    \begin{subfigure}[b]{0.3\textwidth}
        \centering
        \includegraphics[width=\textwidth]{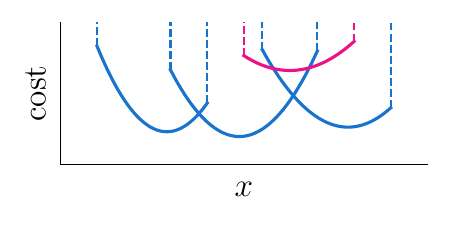}\vspace{-0.3cm}
        \caption{
           $\textsc{RC}=\texttt{False}, \: \textsc{RN}=\texttt{False}$
       }
    \end{subfigure}
    \begin{subfigure}[b]{0.27\textwidth}
        \centering
            \raisebox{5pt}{
        \includegraphics[trim=10 8 10 8, clip, width=\textwidth]{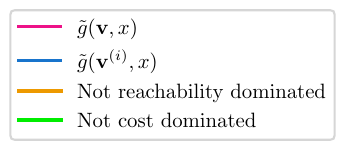}}
        \hfill
    \end{subfigure}
   \caption{Scenarios showing whether a candidate path $\mathbf{v}$ (pink) \textsc{ReachesCheaper} (RC) and \textsc{ReachesNew} (RN) compared to alternate paths $\{\mathbf{v}^{(i)} | i \in \{1,2,3\}\}$ (blue), where the dimension of the set $\mathcal{X}_{\mathbf{v}_{\text{end}}}$ is 1. Green and orange lines denote points $x$ such that \cref{eqn:reaches_cheaper} and \cref{eqn:reaches_new} hold, respectively. Points $x$ for which the cost-to-come $\tilde{g}(\mathbf{v}, x)$ is not drawn (outside dashed lines) are unreachable via path $\mathbf{v}$ due to constraints, and $\tilde{g}(\mathbf{v}, x)=\infty$.}
   \label{fig:comparison}
\end{figure*}

If a candidate path $\mathbf{v}$ reaches some point $x\in \mathcal{X}_{\mathbf{v}_\text{end}}$ cheaper than any way found yet, as in \cref{fig:comparison} (b), (c), and (d), we say it \textsc{ReachesCheaper}:
\begin{equation}
\label{eqn:reaches_cheaper}
\begin{aligned}
\Call{ReachesCheaper}{&\mathbf{v}, S[\mathbf{v}_\text{end}]} \: := \:
\exists \, x \in \mathcal{X}_{\mathbf{v}_\text{end}} \\
    & \text{\st} \:
    \tilde{g}(\mathbf{v}, x) < \tilde{g}(\mathbf{v}', x), \:
    \forall \, \mathbf{v}' \in S[\mathbf{v}_\text{end}].
\end{aligned}
\end{equation}
If $\Call{ReachesCheaper}{\mathbf{v}, S[\mathbf{v}_\text{end}]}= \texttt{False}$, path $\mathbf{v}$ is said to be \emph{dominated}, or, more specifically, \emph{cost-dominated}.
Intuitively, GCS* returns an optimal path when it has pruned only cost-dominated paths (discussed in \cref{sec:algorithm_properties}, proven in \cref{sec:appendix_proofs}) because such paths cannot be subpaths of optimal paths.

If a candidate path $\mathbf{v}$ reaches some point $x\in \mathcal{X}_{\mathbf{v}_\text{end}}$ that has not yet been reached, as in \cref{fig:comparison} (b), we say it \textsc{ReachesNew}:
\begin{equation}
\label{eqn:reaches_new}
\begin{aligned}
\Call{ReachesNew}{&\mathbf{v}, S[\mathbf{v}_\text{end}]} \: := \:
\exists \, x \in \mathcal{X}_{\mathbf{v}_\text{end}} \\
    & \text{\st} \:
    \left( \tilde{g}(\mathbf{v}, x) < \infty \right) \: \wedge \:
    \left( \tilde{g}(\mathbf{v}', x) = \infty, \:
    \forall \, \mathbf{v}' \in S[\mathbf{v}_\text{end}] \right).
\end{aligned}
\end{equation}
If $\Call{ReachesNew}{\mathbf{v}, S[\mathbf{v}_\text{end}]}= \texttt{False}$, $\mathbf{v}$ is said to be \emph{dominated}, or, more specifically, \emph{reachability-dominated}. Where $\textsc{ReachesCheaper}$ compares costs-to-come at all points in the terminal set, $\textsc{ReachesNew}$ compares feasibility of reaching these points.
Note that \textsc{ReachesNew} implies \textsc{ReachesCheaper} since the cost of any feasible path is finite, but \textsc{ReachesCheaper} does not imply \textsc{ReachesNew}. 

\subsubsection{Sampling-based Approximation.}
\label{sec:sampling_based}
The simplest and fastest approach to approximating \textsc{ReachesCheaper} and \textsc{ReachesNew} queries is to check whether sample points $x\in \mathcal{X}_{\mathbf{v}_\text{end}}$ meet the conditions on $x$ in \cref{eqn:reaches_cheaper} and \cref{eqn:reaches_new} respectively, returning \texttt{True} if any sampled point does. 
These approximations are not conservative: the candidate path not being dominated at some sampled point $x$ is sufficient to conclude that the candidate path is not dominated, but not necessary.
For example, in \cref{fig:comparison} (d), sparse sampling may miss the green interval where the cost-to-come of the candidate path $\tilde{g}(\mathbf{v},x)$ is lower than the others, leading $\textsc{ReachesCheaper}$ to return \texttt{False}, signaling that $\mathbf{v}$ is dominated, when in fact it is not.  
However, completeness and optimality are approached in the limit of infinite samples; in this sense we say GCS* is probabilistically complete and asymptotically optimal. We discuss this further in \cref{sec:sampling_properties}.

\subsubsection{Polyhedral-containment-based Approximation.}\label{sec:ah_containment}
We prove in \cref{sec:appendix_proofs} that conservative (necessary but not sufficient) approximations of \textsc{ReachesNew} and \textsc{ReachesCheaper} enable completeness and optimality guarantees.
We formulate such checks via polyhedral containment queries. 
One source of conservatism is that, unfortunately, these queries can check only \textit{single domination}: They compare candidate path $\mathbf{v}$ to individual alternate paths in $S[\mathbf{v}_\text{end}]$.
These checks can detect reachability domination like in \cref{fig:comparison} (a) and (c) and cost domination like in \cref{fig:comparison} (a), in which $\mathbf{v}$ is dominated by a single other path $\mathbf{v}^{(i)}$.  These checks cannot detect reachability domination like in (d) or (e), or cost domination like in (e), in which the candidate path $\mathbf{v}$ is dominated by several alternate paths $\mathbf{v}^{(i)}$ collectively but not by any individual alternate path $\mathbf{v}^{(i)}$.  Further conservatism occurs due to inability to compute exact containment queries tractably.  

We first re-define \textsc{ReachesCheaper} and \textsc{ReachesNew} in terms of set containment, before relaxing these conditions to tractable containment queries. For a path $\mathbf{v}$, we define the set in which a trajectory must reside ($\mathcal{P}_\mathbf{v}$), the set reachable via path $\mathbf{v}$ ($\mathcal{S}_\mathbf{v}$), and the epigraph of the optimal cost-to-come ($\mathcal{C}_\mathbf{v}$):
\begin{subequations}
\begin{equation}
\label{eq:feasible_trajectories}
    \mathcal{P}_\mathbf{v} := 
    \{\mathbf{x} \mid \mathbf{x}_v\in \mathcal{X}_{v}\ \forall \, v \in \mathbf{v}, \:
    (\mathbf{x}_u, \mathbf{x}_{v})\in \mathcal{X}_{(u,v)}\ \forall \, (u,v) \in \mathcal{E}_\mathbf{v}\}.
\end{equation}
\begin{equation}
\label{eq:reachable_set}
    \mathcal{S}_\mathbf{v}:= \{x\in \mathcal{X}_{\mathbf{v}_\text{end}} \mid \exists \, \mathbf{x}\in \mathcal{P}_\mathbf{v} \text{ s.t. } x = \mathbf{x}_{\text{end}}\}\subseteq \mathcal{X}_{\mathbf{v}_\text{end}}.
\end{equation}
\begin{equation}
\label{eq:cost_epigraph}
\mathcal{C}_\mathbf{v}:= 
\{(x,\, l) \in \mathcal{X}_{\mathbf{v}_\text{end}}\times \mathbb{R} \mid \exists \, \mathbf{x}\in \mathcal{P}_\mathbf{v} \text{ s.t. } x = \mathbf{x}_{\text{end}}, \: 
l \geq  \sum_{(u,v)\in \mathcal{E}_\mathbf{v}} c(\mathbf{x}_u,\mathbf{x}_v)\}.
\end{equation} 
\end{subequations}
We can now define the domination queries as set containment queries:
\begin{subequations}
\begin{equation}
\label{eq:reaches_cheaper_contain}
\Call{ReachesCheaper}{\mathbf{v}, S[\mathbf{v}_\text{end}]} = 
\left(\mathcal{C}_\mathbf{v}\nsubseteq \bigcup_{\mathbf{v}'\in S[\mathbf{v}_\text{end}]}\mathcal{C}_{\mathbf{v}'}\right),
\end{equation}
\begin{equation}
\Call{ReachesNew}{\mathbf{v}, S[\mathbf{v}_\text{end}]} = 
\left(\mathcal{S}_\mathbf{v}\nsubseteq \bigcup_{\mathbf{v}'\in S[\mathbf{v}_\text{end}]}\mathcal{S}_{\mathbf{v}'}\right).
\label{eq:reaches_new_contain}
\end{equation}
\end{subequations}
The union over convex sets is not generally itself convex, and we know of no efficient way to compute these queries. Instead, we check single domination, comparing to individual other paths.  In particular, we use

\vspace{-5pt}
\noindent
\begin{subequations}
\begin{minipage}{0.48\textwidth}
\begin{equation}\label{eq:reaches_cheaper_necessary}
    \nexists \mathbf{v}' \in S[\mathbf{v}_\text{end}] \text{ s.t. } \mathcal{C}_\mathbf{v} \subseteq \mathcal{C}_{\mathbf{v}'},
\end{equation}
\end{minipage}
\hfill
\begin{minipage}{0.48\textwidth}
\begin{equation}\label{eq:reaches_new_necessary}
    \nexists \mathbf{v}' \in S[\mathbf{v}_\text{end}] \text{ s.t. } \mathcal{S}_\mathbf{v} \subseteq \mathcal{S}_{\mathbf{v}'},
\end{equation}
\end{minipage}
\end{subequations}
\vspace{8pt}

\noindent as necessary conditions for $\textsc{ReachesCheaper}$ (\ref{eq:reaches_cheaper_contain}) and $\textsc{ReachesNew}$ (\ref{eq:reaches_new_contain}) respectively: if the candidate path is not dominated by the collective of other paths, then it certainly is not dominated by any individual other path.

\looseness=-1These single domination checks are still non-trivial because they involve comparing projections of convex sets, which are expensive to compute explicitly. In our examples, vertex sets are polytopic and costs are linear, which allows us to use Sadraddini and Tedrake's sufficient condition for containment of affine transformations of polyhedra \cite{sadraddiniLinearEncodingsPolytope2019}, evaluated by solving a convex program.  Using this condition to check (\ref{eq:reaches_new_necessary}) and (\ref{eq:reaches_cheaper_necessary}), $\Call{NotDominated}{\mathbf{v}, S[\mathbf{v}_\text{end}]}$ is conservative: If candidate path $\mathbf{v}$ is not dominated, it certainly returns \texttt{True}, but if candidate path $\mathbf{v}$ is dominated, it may return \texttt{True} or \texttt{False}.  In the case of more general convex vertex sets and costs, Jones and Morari's work on computing inner and outer polytopic approximations of convex sets \cite{jonesPolytopicApproximationExplicit2010} can be used before evaluating these containment queries.

\label{sec:domination_checks}

\subsection{Properties of the Algorithm}
\label{sec:algorithm_properties}
We state properties of the algorithm in this section, and include proofs in \cref{sec:appendix_proofs}.  In particular, under conservative \textsc{ReachesCheaper} or \textsc{ReachesNew} checks, GCS* is complete; and under conservative \textsc{ReachesCheaper} checks, GCS* is cost optimal:
\begin{theorem}\label{thm:gcs_star_completeness}
GCS* returns a path from $s$ to $t$ in finite iterations if one exists.
\end{theorem}

\begin{theorem}\label{thm:gcs_star_optimality}
GCS* is \textbf{cost optimal}: GCS* returns an optimal trajectory from $s$ to $t$ in finite iterations if one exists.
\end{theorem}

Note that, like A* \cite{russellArtificialIntelligenceModern2021}, GCS* as stated cannot be assured to terminate in finite iterations in the case where prog.\ (\ref{eq:spp_gcs}) is infeasible and the graph is infinite or cycles are permitted.
GCS* is made complete in the infeasible case by adding a limit on path length to prog.\ (\ref{eq:spp_gcs}) and 
placing a corresponding upper limit on the lengths of paths that get added to $S$ and $Q$ in \cref{line:add_to_S,line:add_to_Q}.

\subsubsection{Efficiency.} 
In contrast to A*, GCS* is not optimally efficient.  That is, given the heuristic information $\Tilde{h}$, it can be shown that GCS* expands subpaths that may not be expanded by some other hypothetical algorithm guaranteed to return the optimal solution.  Conceptually, GCS* loses this property that applies to A* because when GCS* expands a subpath $\mathbf{v}$, by considering paths through the children of $\mathbf{v}_\text{end}$, GCS* now has access to potentially improved lower-bound on the true cost-to-go, $\Tilde{h}(x)$, for $x$ in the reachable set of $\mathbf{v}$, $\mathcal{S}_\mathbf{v}$ (\ref{eq:reachable_set}), but we do not choose to use this information to update $\Tilde{h}(x)$.  A* with a consistent heuristic does not suffer from this inefficiency because it never expands a node twice \cite{dechterGeneralizedBestfirstSearch1985}, whereas GCS* may expand many subpaths whose terminal sets contain points in $\mathcal{S}_\mathbf{v}$.  

\subsubsection{Properties when using sampling-based domination checks.}
\label{sec:sampling_properties}
While we have proven that GCS* is cost-optimal and complete under conservative \textsc{IsDominated} checks, a sampling-based implementation is sometimes preferred, as discussed in \cref{sec:domination_checks}.
In the limit of infinite samples, the sampling-based implementations of \textsc{NotDominated} are exact. Thus it follows from \cref{thm:gcs_star_completeness,thm:gcs_star_optimality} that sampling-based GCS* is probabilistically complete and asymptotically cost-optimal.
One might raise concern that the backward-reachable set from the target in some set $\mathcal{X}_v$ may never be sampled, leading some important feasible path to be pruned. However, because all $\mathcal{P}_\mathbf{v}$ (\ref{eq:feasible_trajectories}) are compact, the sets being checked for containment, $\mathcal{S}_\mathbf{v}$ (\ref{eq:reachable_set}), are closed.  
Therefore, $\mathcal{S}_\mathbf{v}\setminus \bigcup_{\mathbf{v}'\in S[\mathbf{v}_\text{end}]}\mathcal{S}_{\mathbf{v}'}$ with respect to the topology of $\mathcal{S}_\mathbf{v}$ (the space being sampled) is open and thus either empty or positive-measure.  If it has positive measure, it will eventually be sampled.

\vspace{-10pt}
\section{Application to Planar Pushing}\label{sec:planar_pushing}

Tasks involving making and breaking contact while constraining non-penetration notoriously lead to an explosion of discrete modes, especially as the number of bodies scales.  Indeed, the naive formulation in \cref{sec:formulation} sees this combinatorial growth. With more careful construction of contact modes \cite{huangEfficientContactMode2021}, or by introducing hierarchy \cite{chengEnhancingDexterityRobotic2023}, one could greatly reduce search space for these problems. However, additional structure can make such methods less general. Our goal with these experiments is not to demonstrate state of the art performance on contact-rich manipulation tasks specifically, but instead to show that our method is able to approach such large problems.

\subsection{Formulation}\label{sec:formulation}
We implement a simplified planar pushing model with polyhedral robots and objects that can translate but not rotate, and static obstacles.  Robot-object and object-object contact are frictionless, and motion quasi-static, where bodies' velocities are proportional to the net forces applied to them.  We enforce non-penetration between all pairs of bodies, defining collision-free sets over robot and object positions.  We also define sets over positions and contact forces defining the permissible physical contact behavior for each pair of bodies.  Together, these two kinds of discrete modes define the sets of the graph $G$.  Each point $\mathbf{x}_v\in\mathcal{X}_v$ in prog.\ \ref{eq:spp_gcs} is comprised of $n_k$ knot points, defining the planar positions for each robot and object, actuation forces for each robot, and a force magnitude for each pair of bodies in contact. The constraints defining these sets, which we explain throughout this section, are straightforward to compose.  As such, the \textsc{Successors} operator can construct them on demand. In particular, the successors of some vertex are defined by changing the contact state or the non-penetration separating hyperplane (defined below) for a single pair of bodies.  

All constraints are linear, leading to polyhedral sets.  However, extensions of this model to include rotations could involve semidefinite relaxations, leading to spectrahedral sets, in a formulation similar to that of Graesdal et al.\ \cite{graesdalTightConvexRelaxations2024}.  This work could also be extended to borrow other components of the GCS formulation of planar pushing from Graesdal et al., like frictional contact.   We omit these features for simplicity.  However, the assumption that bodies are polyhedral is harder to relax, as it allows for an exact polyhedral decomposition of collision-free space. Alternative  settings that require only an approximate convex decomposition of the free configuration space can handle rotations and non-polyhedral bodies \cite{petersenGrowingConvexCollisionFree2023,wernerApproximatingRobotConfiguration2024,daiCertifiedPolyhedralDecompositions2023}.  However, these formulations are not conducive to making contact due to incomplete coverage.  As these decompositions are generally composed offline and cannot produce new sets as quickly as our \textsc{Successors} operator can, this formulation would either require building the entire graph before running GCS* or accepting slower evaluation of \textsc{Successors}.

\subsubsection{Non-penetration.}
We constrain that each pair of bodies do not penetrate one another by imposing that one face of one of the two polyhedra is a separating hyperplane: all vertices of one body lie on one side, and all vertices of the other body lie on the other side.

\subsubsection{Contact.}
Every pair of bodies may have any of the following kinds of contact: no contact, any face of one body in contact with any face of another, or any face of one body in contact with any vertex of another.  Gathering these options over all pairs of bodies produces a \textit{contact set}.  Each pair of bodies in contact in each contact set is accompanied by a force magnitude variable, leading to variation in dimension between sets.  This force acts normal to the face in contact. Position constraints are added to ensure the associated contact is physically feasible. Relative sliding of features in contact is allowed within a single contact set, as long as they remain in contact.
\subsubsection{Quasi-static dynamics.}
\looseness=-1 For each object and robot, we enforce that between  consecutive knot points, translation is proportional to the sum of forces on the body, including a robot's actuation force.  Edge constraints enforce that robot and object positions remain the same across edges -- for an edge $(u, v)$ in \cref{eq:spp_gcs}, the positions from the last knot point in $\mathbf{x}_u$ must equal the positions from the first knot point in $\mathbf{x}_v$.  As a result of this constraint, a path containing the edge $(u, v)$ is only feasible if there is a shared position that is feasible in both convex sets associated with these vertices, $\mathcal{X}_u$ and $\mathcal{X}_v$.

\subsection{Implementation Details}
\looseness=-1We use $n_k = 2$ knot points per set.   Our edge cost is the $L_1$ norm distance between knot points, plus a constant penalty for switching modes. That is, for $v\neq t$, $c(\mathbf{x}_u,\mathbf{x}_v) := 1 + \sum_i w_i||p^i_{v,1} - p^i_{v,0}||_1$, for weights $w_i$, where $p^i_{v,j}$ gives the position of the $i$th moving body (object or robot) at the $j$th knot point in vertex $v$.  $c(\mathbf{x}_u, t) = 0$ (no movement in the target set).  We use $w_i=1$ for all $i$.
Our sampling-based implementations can handle $L_2$ norm or $L_2$ norm squared edge costs easily. However, as discussed in \cref{sec:ah_containment}, our containment-based implementations need modification to accommodate non-linear costs.

\looseness=-1 For simplicity, our nominal $\tilde{h}$ is the cost of a ``shortcut'' edge between the point $x$ and some point in the target set $x_t \in \mathcal{X}_t$ chosen to minimize $\tilde{h}$.  If an edge does exist directly to the target, the true cost of that edge is used.  Otherwise, $\tilde{h}(x) \gets 1 + c(x,x_t)$, using a weight of $0.2$ for robot displacements instead of the true weights $w_i = 1$.  Because the weights do not exceed the true weights, this yields an admissible heuristic. Note that $1$ is added to $c(x,x_t)$ because we know that a constant penalty of $1$ is applied for every mode switch and there will be at least one additional mode switch before reaching the target.  For speed, we also conduct experiments with an $\epsilon$\textit{-suboptimal} heuristic instead, where we scale $\tilde{h}$ by some $\epsilon$, $\tilde{h}(x) \gets \epsilon \cdot (1+c(x,x_t))$, to prune more aggressively, using $\epsilon = 10$. The shortcut edge ignores contact dynamics, making this heuristic less informative in some scenarios than others. $\tilde{h}$ can be further optimized, but this is not the focus of this paper. Practitioners may apply domain knowledge to formulate informative heuristics, trading off between informativeness and computational cost.

For sampling-based domination checks, we use only a single sample per check. As we will discuss, empirically, this is sufficient.  In practice, due to edge constraints, the reachable set for a path $\mathbf{v}$, $\mathcal{S}_\mathbf{v}\subseteq \mathcal{X}_{\mathbf{v}_\text{end}}$, is often low-volume within the terminal set $\mathcal{X}_{\mathbf{v}_\text{end}}$.  As such, we sample uniformly in $\mathcal{X}_{\mathbf{v}_\text{end}}$ and then project onto $\mathcal{S}_\mathbf{v}$.

For containment-based methods, we first perform a cheap single-sample \textsc{NotDominated} check, and only check containment if it returns \texttt{False}.
Additionally, taking advantage of the unique structure of our planar pushing formulation, we only check the domination conditions on the last position knot point instead of all variables in the set. To reduce problem size for domination checks, we ``solve away'' equality constraints by parameterizing in their nullspace.

\vspace{-10pt}
\section{Results}
\label{sec:results}
\vspace{-5pt}
\begin{figure}[ht]
    \begin{subfigure}[t]{1.0\linewidth}
        \centering
        \includegraphics[trim=0 60 0 10, clip, width=\linewidth]{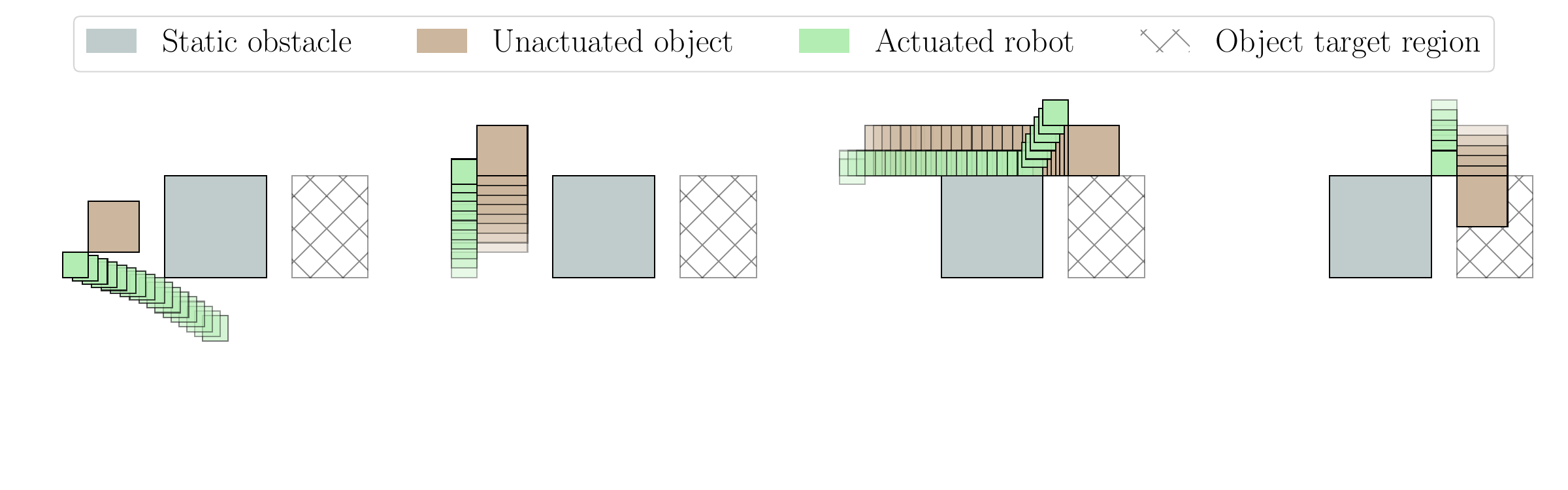}
        \begin{picture}(0,0)
            \put(-175,75){\small (a)}
          \end{picture}
    \end{subfigure}
    \begin{subfigure}[t]{1.0\linewidth}
        \vspace{-1em}
        \centering
        \includegraphics[trim=40 20 40 20, clip, width=\linewidth]{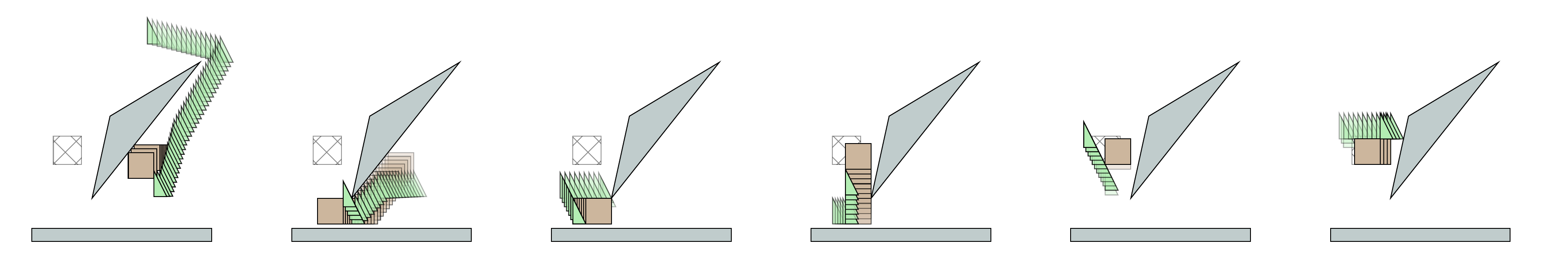}
        \begin{picture}(0,0)
            \put(-175,60){\small (b)}
          \end{picture}
    \end{subfigure}
    \caption{The $\epsilon$-suboptimal solution trajectories found by GCS* using the sampling-based \textsc{ReachesCheaper} domination check on the AROUND Task (a) and SQUEEZE Task (b). 
    In SQUEEZE, in the second frame, the robot pushes the object past the obstacle to make room for itself to maneuver around in the third frame.  In the fourth, the robot uses contact with the obstacle to slide the object vertically until the robot can fit underneath it.
    Videos of these trajectories are available on the 
    \href{https://shaoyuan.cc/research/gcs-star/}
    {project website}.}
    \label{fig:trajectories}
\end{figure}
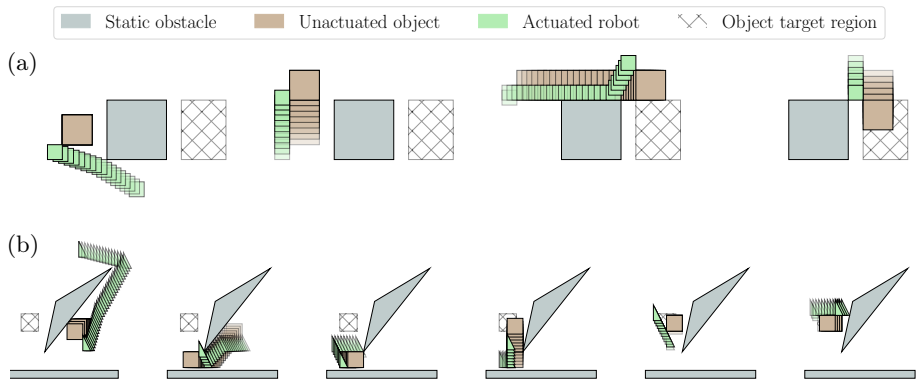
\Cref{tab:task_results} presents results for three planar pushing tasks: AROUND, SQUEEZE, and STACK shown in \cref{fig:trajectories,fig:hero}.
AROUND is the simplest task: its graph has 194 vertices and 8,328 edges. SQUEEZE has a larger graph with 628 vertices and 56,806 edges. The dimensions of sets in AROUND and SQUEEZE range from 4 to 12. STACK (\cref{fig:hero}) has 5 bodies, leading to a combinatorial explosion of  approximately $1.3 \times 10^9$ sets and up to $8.5 \times 10^{17}$ edges. The dimensions of its sets range from 8 to 24. 
Due to the sizes of the graphs, the direct convex optimization approach \cite{marcucciShortestPathsGraphs2024} cannot be used, thus we compare GCS* against IxG and IxG* \cite{natarajanImplicitGraphSearch2024}, the state-of-the-art for incrementally solving SPP on GCS. 

\renewcommand{\arraystretch}{1.1}  
\begin{table}[htb]
	\centering
    \caption{Here we show the solve time, solution cost, and number of expanded paths (Exp) for three planar pushing tasks.
    Our algorithm GCS* is compared to the baselines IxG* and IxG (grey).
    Sampling (Sampl.) and containment (Cont.) based implementations (Impl.) of the domination checkers (DC) $\textsc{ReachesCheaper}$ (RC) and $\textsc{ReachesNew}$ (RN) are shown.
    A dagger ($\dagger$) indicates that the property holds probabilistically or asymptotically. Some experiments returned infeasible (\texttt{Fail}), were stopped after 10h ($\times$), or could not be run as the graph could not be explicitly constructed (-).
    The baselines required preprocessing times of 1.1m and 1.7h for AROUND and SQUEEZE respectively.
    Our methods require no preprocessing time.
    All computation was performed on an AMD Ryzen 9 7950x 16-core 32-thread processor.
    \vspace{.2cm}
    }
    \fontsize{7pt}{8pt}\selectfont
    \begin{tabular}{|c|c|c|c|c|c|c|c|c||c|c|c||c|c|c|}
        \hline
        \multirow{2}{*}{Alg.} & \multirow{2}{*}{DC} & \multirow{2}{*}{Impl.} & \multirow{2}{*}{Heur.} 
        & \multirow{2}{*}{\shortstack{Com-\\plete}} & \multirow{2}{*}{\shortstack{Cost\\optimal}} 
        & \multicolumn{3}{c||}{AROUND} & \multicolumn{3}{c||}{SQUEEZE} & \multicolumn{3}{c|}{STACK} \\ 
        \cline{7-15}
        & & & & & & Time & Cost & Exp & Time & Cost & Exp & Time & Cost & Exp \\
        \hline
        GCS* & RC & Sampl. & Admiss. &
        Yes$^\dagger$ & Opt.$^\dagger$
        & 26.2s &  27.5 & 432 & 4m & 47.80 & 2114 & $\times$ & $\times$ & $\times$ \\
        GCS* & RC & Cont. & Admiss. &
        Yes & Opt.
        & 7.5h & 27.5 & 734  & $\times$ & $\times$ & $\times$ & $\times$ & $\times$ & $\times$\\
        \color{gray} IxG* & \color{gray} N.A.& \color{gray} N.A. & \color{gray} Admiss.
        & \color{gray} Yes & \color{gray} Opt.
        & \color{gray} $\times$ & \color{gray} $\times$ & \color{gray} $\times$ & \color{gray} $\times$ & \color{gray} $\times$ & \color{gray} $\times$ & \color{gray} - & \color{gray} - & \color{gray}- \\
        \hline
        GCS* & RC & Sampl. & Inadm. &
        Yes$^\dagger$ & $\epsilon$-subopt.$^\dagger$
        & 2.9s & 27.5 & 51 & 38.3s & 54.55 & 381 & 21.9s & 67.53 & 57 \\
        GCS* & RC & Cont. & Inadm. &
        Yes & $\epsilon$-subopt.
        &  4.6m & 27.5 & 73 & $\times$ & $\times$ & $\times$ & $\times$ & $\times$ & $\times$ \\
        \color{gray} IxG* & \color{gray} N.A.& \color{gray} N.A. & \color{gray} Inadm. & 
        Yes & \color{gray} $\epsilon$-subopt.
        & \color{gray} $\times$ & \color{gray} $\times$ & \color{gray} $\times$ & \color{gray} $\times$ & \color{gray} $\times$ & \color{gray} $\times$ & \color{gray} - & \color{gray} - & \color{gray}-  \\
        \hline
        GCS* & RN & Sampl. & Inadm. &
        Yes$^\dagger$ & Subopt.
        & 2.5s&  27.5 & 44 & 20s & 57.42 & 259 & 21.8s& 67.53 & 57 \\
        GCS* & RN & Cont. & Inadm. &
        Yes & Subopt.
        & 35.4s &  28.5 & 75 & 20.5m & 57.42 & 458 & 3.5m & 67.53 & 64 \\
        \color{gray} IxG & \color{gray} N.A.& \color{gray} N.A.  & \color{gray} Inadm. & 
        \color{gray} No & \color{gray} Subopt.
        & \color{gray} 4.7s & \color{gray} 28 & \color{gray} 31 & \color{gray} 2.3m & \color{gray} \texttt{Fail} & \color{gray} 615 & \color{gray} - & \color{gray} - & \color{gray} - \\
        \hline
    \end{tabular}
    \label{tab:task_results}
    \vspace{-.1cm}
\end{table}

For sampling-based domination checks, as we increase the number of samples per check, we approach completeness and optimality. 
Due to the curse of dimensionality,  such coverage demands exponentially more samples as set dimensions increase. However, in our tasks, we found that very sparse sampling -- in particular, a single sample per domination check -- still resulted in good solution quality. This is an interesting result that enables an efficient and de facto nearly optimal implementation.
We understand this to be a property particular to our problem structure, in which the condition that two paths reach the same final knot point within the same vertex is quite restrictive.  For example, the sampling-based \textsc{ReachesNew} approximation asks whether any path in $S$ can achieve a trajectory ending at a particular sampled object position and robot position in some contact state that can be achieved by a candidate path, calling the candidate path dominated if so.
In this event, it is likely that that path in $S$ can reach much of what the candidate path can, and the candidate path can be pruned with little consequence.

Containment-based domination checks scale poorly with path length, or, more precisely, the dimension of the full trajectory through it.  This is reflected in the solve times in the rows marked ``Cont.'' in \cref{tab:task_results}.

\looseness=-1The number of paths expanded is impacted greatly by the strength of the heuristic in the context of the particular task. In general, the ``shortcut'' edge heuristic is strong for AROUND and STACK and weaker for SQUEEZE which requires squeezing through a narrow gap then maneuvering around the object to push it from a different direction multiple times.
As expected, use of an inadmissible heuristic leads to cost increase. 
Typically, containment-based checks lead to more paths being expanded as compared to sampling-based checks (as seen in \cref{tab:task_results}) because of false negatives from sampling, as well as conservatism of the single domination criteria. However, containment-based checks could also lead to fewer paths being expanded because important paths are not wrongly pruned, as they might be when using sampling-based checks (not observed in \cref{tab:task_results}).

In comparison to our methods, the baselines require significant time to compute heuristic values for the entire graph before starting their query phase. While this approach can yield a stronger heuristic (as in the bottommost row, where IxG expands fewer paths than GCS* for AROUND), for these tasks the solve times can be dominated by this pre-processing phase. When the graph is very large, this pre-processing can become intractable (as in STACK). 
Because IxG* prunes paths based on a global upper bound found by IxG, in cases where IxG is unable to return a solution (as in SQUEEZE), or simply when an admissible heuristic is used (as in row 3), no paths will be pruned, resulting in intractably large search spaces.
As GCS* is complete, and our domination checks do not rely on a global upper bound, GCS* does not experience these particular limitations.

\vspace{-10pt}
\section{Conclusion}
\label{sec:conclusion}
\vspace{-5pt}
We propose GCS*, a forward heuristic search algorithm for solving large discrete-continuous planning problems formulated as GCS. 
We define two domination checks \textsc{ReachesNew} and \textsc{ReachesCheaper}, as well as containment and \linebreak sampling-based implementations of those checks that allow GCS* to be complete and optimal, or have probabilistic/asymptotic versions of those properties, respectively.
GCS* provides a principled way of addressing the challenges of applying graph search to the discrete-continuous setting.
We demonstrate settings in which GCS* performs favourably compared to the state-of-the-art.
For application to real-world planar pushing tasks, further work would be required to handle rotations \cite{graesdalTightConvexRelaxations2024}. Incorporating these insights into algorithms that leverage hierarchy, factorization or learned heuristics could solve more complex problems faster.
\vspace{-5pt}

\begin{credits}
\subsubsection{\ackname}
This work was supported by the Aker Scholarship; Amazon.com, PO No. 2D-12585006; The AI Institute, award ID Agmd Dtd 8/1/2023; ARO grant W911NF-23-1-0034; and The Charles Stark Draper Laboratory, Inc., where Rebecca H. Jiang is a Draper Scholar.
\vspace{-5pt}
\subsubsection{\discintname}
The authors have no competing interests to declare that are
relevant to the content of this article.
\end{credits}
%
%
%
\bibliographystyle{splncs04}
\bibliography{references}
\appendix
\section{Proofs of Algorithmic properties}\label{sec:appendix_proofs}
\setcounter{theorem}{0}
We prove that, under conservative \textsc{NotDominated} or specifically \textsc{ ReachesCheaper } checks, GCS* is complete or optimal, respectively: The proofs rely on these checks  being necessary but not sufficient conditions for \cref{eqn:reaches_new} and \cref{eqn:reaches_cheaper}.  Optimality also relies on $\tilde{h}$ being pointwise admissible (\cref{def:pointwise_admissible}).
\subsubsection{Completeness.}\label{sec:completeness}
We show that if prog.\ (\ref{eq:spp_gcs}) is feasible, GCS* returns a feasible path in finite iterations.  However, like A* \cite{russellArtificialIntelligenceModern2021}, GCS* as stated cannot be assured to terminate in finite iterations in the case where prog.\ (\ref{eq:spp_gcs}) is infeasible and the graph is infinite or cycles are permitted.
GCS* is made complete in the infeasible case by adding a limit on path length to prog.\ (\ref{eq:spp_gcs}) and 
placing an upper limit on the lengths of paths that get added to $S$ and $Q$ in \cref{line:add_to_S,line:add_to_Q}.

\vspace{-3pt}
\begin{lemma}\label{lemma:finite_expansion}
    For any path $\mathbf{v}$ in $Q$, GCS* either expands $\mathbf{v}$ after a finite number of iterations or returns a feasible path to the target first.
\end{lemma}
\begin{proof}
    Path $\mathbf{v}$ has some finite estimated total cost $\tilde{f}(\mathbf{v})$.  Because all edge costs are positive and bounded away from zero, there exist a finite number of paths $\mathbf{v}'$ with cost-to-come that do not exceed $\tilde{f}(\mathbf{v})$, $\tilde{g}(\mathbf{v}')\leq \tilde{f}(\mathbf{v})$.  As $\tilde{f}(\mathbf{v}')\geq \tilde{g}(\mathbf{v}')$ for any $\mathbf{v}'$, there are a finite number of paths $\mathbf{v}'$ ever added to $Q$ such that $\tilde{f}(\mathbf{v}')\leq \tilde{f}(\mathbf{v})$.  As GCS* always expands the path in $Q$ with the lowest $\tilde{f}$ value, GCS* will eventually exhaust all such paths and expand $\mathbf{v}$, unless it terminates with a feasible path to the target first.
\end{proof}

\vspace{-10pt}
\begin{lemma}
\label{lemma:reachable_visited}
If target $t$ is reachable from source $s$, then some path that reaches it is added to $Q$ in finite iterations.
\end{lemma}

\begin{proof}
Suppose, for contradiction, that $t$ is reachable from $s$, but no path that reaches it is added to $Q$ in finite iterations.  There is some feasible path $\mathbf{v}$ where $\mathbf{v}_\text{start} = s$, $\mathbf{v}_\text{end} = t$, and some feasible trajectory $\mathbf{x}\in \mathcal{P}_\mathbf{v}$.  Let $j$ be the largest index such that there exists a path $\mathbf{v}'$ (which may or may not be a subpath of $\mathbf{v}$) terminating in $\mathbf{v}_j$ ($\mathbf{v}'_\text{end} = \mathbf{v}_j$) and reaching $\mathbf{x}_{j}$, that ever gets added to $Q$.  Such a $j\geq 0$ exists because $Q$ is initialized with the path containing just $s$. 
By \cref{lemma:finite_expansion}, $\mathbf{v}'$ is expanded in finite iterations, unless a path reaching $t$ is returned first.  If a path reaching $t$ is returned first, we achieve contradiction trivially.  Consider instead the case where $\mathbf{v}'$ is expanded in finite iterations.
By assumption, $\mathbf{v}_{j+1}$ is a successor of $\mathbf{v}_{j}$, and so  the extended path $[\mathbf{v}', \mathbf{v}_{j+1}]$ must be added to $Q$ when $\mathbf{v}'$ is expanded.  We can be sure  it is added to $Q$ instead of being rejected by \textsc{NotDominated} in line 10 because
 $(\mathbf{x}_{j}, \mathbf{x}_{j+1})\in \mathcal{X}_{(\mathbf{v}_{j}, \mathbf{v}_{j+1})}$, making $\mathbf{x}_{j+1}$ reachable via the extended path $[\mathbf{v}', \mathbf{v}_{j+1}]$, and, by assumption, no path terminating in $\mathbf{v}_{j+1}$ and reaching $\mathbf{x}_{j+1}$ has been added to $Q$ prior, and thus nor has it been added to $S[\mathbf{v}_{j+1}]$, as $S$ is only added to in conjunction with $Q$.  However, the extended path $[\mathbf{v}', \mathbf{v}_{j+1}]$ is now a path in $Q$ terminating in $\mathbf{v}_{j+1}$ and reaching $\mathbf{x}_{j+1}$, creating a contradiction with the definition of $j$. 
\end{proof}

\begin{theorem}\label{thm:gcs_star_completeness}
GCS* returns a path from $s$ to $t$ in finite iterations if one exists.
\end{theorem}
\begin{proof}
By \cref{lemma:reachable_visited}, a feasible path from $s$ to $t$ is added to $Q$ in finite iterations if one exists.  By  \cref{lemma:finite_expansion}, either this path is expanded in finite iterations (in which case it is immediately returned), or another such feasible path from $s$ to $t$ is returned first.  In both cases, a path from $s$ to $t$ is returned in finite iterations.
\end{proof}

\subsubsection{Cost-optimality.}\label{sec:optimality}
In this section, we assume \textsc{NotDominated} is an exact or conservative check for \textsc{ReachesCheaper}.

\begin{lemma}
\label{lemma:Q_has_node_with_optimal_cost_to_come}
Suppose prog.\ (\ref{eq:spp_gcs}) is feasible and $\mathbf{v}$ is an optimal path from $s$ to $t$.  If $t$ is unexpanded, there is a path $\mathbf{v}'$ in $Q$ with $\mathbf{v}'_{\text{end}} = \mathbf{v}_i$ for some $i$, such that $\mathbf{v}'$ attains the optimal cost-to-come to $\mathbf{x}^*(\mathbf{v})_{i}$, $\tilde{g}(\mathbf{v}', \mathbf{x}^*(\mathbf{v})_{i}) = g^*(\mathbf{x}^*(\mathbf{v})_{i})$.
\end{lemma}
\begin{proof}
    If GCS* has not completed one iteration, the lemma holds trivially via the path containing only $s$, $\mathbf{v}^s:= [s]\in Q$.   
    Suppose instead $\mathbf{v}^s$ has been expanded.  
    Let $\mathbf{v}$ be any optimal path from $s$ to $t$, and let $\mathbf{x}:=\mathbf{x}^*(\mathbf{v})$ be an optimal trajectory through $\mathbf{v}$.
    Let $i$ be the 
    largest index for which there exists an expanded path $\mathbf{v}'$ such that $\mathbf{v}_i = \mathbf{v}'_{\text{end}}$, with optimal cost-to-come to $\mathbf{x}_i$, $\tilde{g}(\mathbf{v}', \mathbf{x}_i) = g^*(\mathbf{x}_i)$.
    That is, the $i$th vertex on $\mathbf{v}$ has been expanded on some path 
    $\mathbf{v}'$
    which may or may not be a subpath of $\mathbf{v}$ but attains the same (optimal) cost-to-come to $\mathbf{x}_i$.  Since $s$ has been expanded, and $t$ has not, $i\in [0, \text{end}-1]$.  
    Because $(\mathbf{x}_i, \mathbf{x}_{i+1})\in \mathcal{X}_{(\mathbf{v}_i, \mathbf{v}_{i+1})}$ (we know this because the optimal trajectory $\mathbf{x}$ traverses this edge via these points), 
   the extended path $[\mathbf{v}', \mathbf{v}_{i+1}]$
    reaches $\mathbf{x}_{i+1}$ with optimal cost-to-come $g^*(\mathbf{x}_{i+1}) = g^*(\mathbf{x}_{i}) + c(\mathbf{x}_i, \mathbf{x}_{i+1})$.
    
    We now argue that once $\mathbf{v}'$ was expanded, there must have been, at some point present or prior, a path added to $Q$ terminating at $\mathbf{v}_{i+1}$ and reaching $\mathbf{x}_{i+1}$ with optimal cost-to-come. Then, we argue that it is still in $Q$.
    When $\mathbf{v}'$ was expanded, the extended path $[\mathbf{v}', \mathbf{v}_{i+1}]$ was placed on $Q$ unless it was rejected by \textsc{ReachesCheaper} returning \texttt{False} in line 10.  
    If \textsc{ReachesCheaper} returned \texttt{False}, then another path was already in $S[\mathbf{v}_{i+1}]$ that terminated at $\mathbf{v}_{i+1}$ and reached $\mathbf{x}_{i+1}$ with optimal cost-to-come.  Paths are added to $S$ only when added to $Q$, so this path was added to $Q$.
    By the definition of $i$, whichever such path was at any point in $Q$ has not been expanded and is still in $Q$. 
\end{proof}

\begin{corollary}
\label{corollary:unexpanded_t_Q_has_optimal_node}
Suppose prog.\ (\ref{eq:spp_gcs}) is feasible, and $\mathbf{v}$ is an optimal path from $s$ to $t$. If \cref{alg:gcs_star} has not terminated, there is some $\mathbf{v}'$ in $Q$ such that $\mathbf{v}'_\text{end}\in \mathbf{v}$ and the estimated total cost for $\mathbf{v}'$ does not exceed the true optimal cost, $\tilde{f}(\mathbf{v}')\leq f^*(s)$.
\end{corollary}

\begin{proof}
As \cref{alg:gcs_star} terminates only when $t$ is expanded, $t$ is unexpanded under these assumptions. By lemma \ref{lemma:Q_has_node_with_optimal_cost_to_come}, there is an unexpanded search path $\mathbf{v}'$  in $Q$ with $\mathbf{v}'_\text{end} = \mathbf{v}_i$ for some $i$ and an optimal cost-to-come to $\mathbf{x}^*(\mathbf{v})_{i}$. 

We know $\tilde{f}(\mathbf{v}') \leq  \tilde{f}(\mathbf{v}', \mathbf{x}^*(\mathbf{v})_{i})$ because,
more generally speaking, a path constrained to go through $\mathbf{v}'$ is estimated to be at least as cheap as a path additionally constrained to go through any particular point.
Then, 
\begin{align*}
         \tilde{f}(\mathbf{v}', \mathbf{x}^*(\mathbf{v})_{i}) &= \tilde{g}(\mathbf{v}', \mathbf{x}^*(\mathbf{v})_{i}) + \tilde{h}(\mathbf{x}^*(\mathbf{v})_{i}) \quad\textnormal{(By definition of $\tilde{f}$)} \\
        &= g^*(\mathbf{x}^*(\mathbf{v})_{i}) + \tilde{h}(\mathbf{x}^*(\mathbf{v})_{i}) \quad \textnormal{(By lemma \ref{lemma:Q_has_node_with_optimal_cost_to_come})} \\
        &\leq g^*(\mathbf{x}^*(\mathbf{v})_{i}) + {h}^*(\mathbf{x}^*(\mathbf{v})_{i}) \quad \textnormal{(By $\tilde{h}$ pointwise admissibility)} \\
        &=f^*(s) \quad\textnormal{(Since $\mathbf{x}^*(\mathbf{v})_{i}$ is on an optimal trajectory).}
    \end{align*}
Thus, we have $\tilde{f}(\mathbf{v}')\leq f^*(s).$ 
\end{proof}

\begin{theorem}\label{thm:gcs_star_optimality}
GCS* is \textbf{cost optimal}: GCS* returns an optimal trajectory from $s$ to $t$ in finite iterations if one exists.
\end{theorem}

\begin{proof}
By \cref{thm:gcs_star_completeness}, GCS* returns a feasible path $\mathbf{v}'$ from $s$ to $t$ in finite iterations.  Suppose for contradiction that the cost-to-come to $t$ through $\mathbf{v}'$ is suboptimal, $\tilde{g}(\mathbf{v}', \mathbf{x}^*(\mathbf{v}')_\text{end}) > f^*(s)$.  By admissibility, $h^*(x) = \tilde{h}(x) = 0\ \forall x\in \mathcal{X}_t$.  Thus, $\tilde{f}(\mathbf{v}') = \tilde{g}(\mathbf{v}', \mathbf{x}^*(\mathbf{v}')_\text{end}) > f^*(s)$.  By corollary \ref{corollary:unexpanded_t_Q_has_optimal_node}, just before termination, 
there existed a path $\mathbf{v}$ in $Q$ with $\tilde{f}(\mathbf{v})\leq f^*(s) < \tilde{f}(\mathbf{v}')$.  However, \cref{alg:gcs_star} always pops the path in $Q$ with the lowest $\tilde{f}$ value, and returns a path only after popping it.  This contradicts the assumption that $\mathbf{v}'$ was returned.  

\end{proof}

\end{document}